\documentclass[11pt,reqno]{amsart}
\usepackage[top=1.2in, bottom=1.2in, left=1.3in, right=1.3in]{geometry}
\usepackage{amsfonts, amssymb, amsmath, enumerate, bm, xspace}
\usepackage{setspace}

\usepackage{epsfig,subfigure}
\graphicspath{{figures/}}
\usepackage{epstopdf}
\usepackage{caption}
\usepackage{bbm,mathtools}
\usepackage{floatrow,pdfsync} %

\numberwithin{equation}{section}

\DeclareMathOperator{\E}{\mathbb{E}}

\newcommand{\ip}[2]{\langle#1,#2\rangle}

\def \C {\mathbb{C}}

\def \P {\mathbb{P}}
\def \R {\mathbb{R}}

\def \l {\lambda}

\newtheorem{theorem}{Theorem}[section]

\newtheorem{lemma}[theorem]{Lemma}

\theoremstyle{remark}

\newtheorem{assumption}{Assumption}[section]

\begin{document}

\title[]{Estimating the number of communities by spectral methods}

\author{Can M. Le}
\address{Department of Statistics, University of California, Davis, CA 95616}
\email{canle@ucdavis.edu}
\author{Elizaveta Levina}
\address{Department of Statistics, University of Michigan, Ann Arbor, MI 48109}
\email{elevina@umich.edu}

\thanks{}

\begin{abstract}
Community detection is a fundamental problem in network analysis with many methods available to estimate communities.
Most of these methods assume that the number of communities is known, which is often not the case in practice.
We study a simple and very fast method for estimating the number of communities
based on the spectral properties of certain graph operators, such as the non-backtracking matrix and the Bethe Hessian matrix.
We show that the method performs well under several models and a wide range of parameters, and is guaranteed to be consistent under several asymptotic regimes.
We compare this method to several existing methods for estimating the number of communities and show that it is both more accurate and more computationally efficient.
\end{abstract}

\maketitle

\section{Introduction}
The problem of clustering similar objects into groups is a fundamental problem in data analysis.
In network analysis, it is known as community detection (\cite{Newman&Girvan2004,Airoldi2008,Bickel&Chen2009,amini2013pseudo}).
Given a network, which consists of a set of nodes and a set of edges between them, the goal of community detection is to
cluster the nodes into groups (communities) so that nodes in the same community share a similar connectivity.

One of the simplest ways of modeling a community structure is the stochastic block model (SBM), proposed by \cite{Holland83}.
Given the number of communities $K$, $n$ node labels $c_i$ are drawn independently from a multinomial distribution with parameter $\pi = (\pi_1,...,\pi_K)$.
The edges between pairs of nodes $(i,j)$ are then drawn independently from a Bernoulli distribution with parameter $P_{c_ic_j}$ and collected in the $n \times n$ adjacency matrix $A$, with $A_{ij} = 1$ if nodes $i$ and $j$ are connected by an edge, and 0 otherwise.
A limitation of the stochastic block model is that all nodes in the same communities are equivalent and follow the same degree distribution, whereas many real networks contain a small number of high-degree nodes, the so called hubs.
To address this limitation, \cite{Karrer10} proposed the degree-corrected stochastic block model (DCSBM).
It  assigns a degree parameter $\theta_i$ to each node $i$, and edges between nodes are drawn independently with probabilities $\theta_i\theta_j P_{c_ic_j}$.
The community detection task is to recover the labels $c_i$ given the adjacency matrix $A$.

A large number of methods have been proposed for finding the underlying community structure
(\cite{McS01,NewmanPNAS,Airoldi2008,Bickel&Chen2009,Rohe2011,Chaudhuri&Chung&Tsiatas2012,amini2013pseudo,
Krzakala.et.al2013spectral,Vu2014,Mossel&Neeman&Sly2014a,Saade&Krzakala&Zdeborova2014}).
Most of these methods require the number of communities $K$ as input, but in practice $K$ is often unknown.
To address this problem, a few likelihood-based methods have been proposed to estimate $K$ under either the SBM or the DCSBM
(\cite{Daudinetal2008,Latouche&Birmelel&Ambroise2012,Peixoto2013,Saldana&Yu&Feng2014,Wang&Bickel2015}).  These methods use BIC-type criteria for choosing the number of communities from a set of possible values, which requires computing the likelihood, done using either MCMC or the variational method, which are both computationally very challenging for large networks.
A different approach based on the distribution of leading eigenvalues of an appropriately scaled version of the adjacency matrix was proposed by \cite{Bickel&Sarkar2013,Lei.goodness.of.fit2014}.   Under the SBM, distributions of the leading eigenvalues converge to the Tracy-Widom distribution;
this fact is used to determine $K$ through a sequence of hypothesis tests.   Since the rate of convergence is slow for relatively sparse networks, a bootstrap correction procedure was employed, which also leads to a high computational cost.
Cross-validation approaches were proposed by \cite{Chen&Lei2014} and \cite{Li&Levina&Zhu.NCV2016}.  While they have good properties under the SBM and the DCSBM, they require estimating communities on many random network splits, and are computationally costly.  

To the best of our knowledge, all existing methods are either restricted to a specific model or computationally intensive.
In this paper we study a fast and reliable method that uses spectral properties of either the Bethe Hessian or the non-backtracking matrices.
Under a simple SBM in the sparse regime, these matrices have been used to recover the community structure
(\cite{Krzakala.et.al2013spectral,Saade&Krzakala&Zdeborova2014,Bordenave.et.al2015non-backtracking});
It was observed in the physics literature that the informative eigenvalues (i.e., those corresponding to eigenvectors which encode the community structure) of these matrices are well separated from the bulk and can be used to estimate the number of communities, but the properties of this estimator have never been investigated, either theoretically or empirically.  
We show that the number of ``informative'' (to be defined explicitly below)  eigenvalues of these matrices directly estimates the number of communities, and the estimate performs well under different network models and over a wide range of parameter values, outperforming existing methods designed specifically for estimating $K$ under either SBM or DCSBM.   This method is extremely computationally efficient, since all it requires is computing a few leading eigenvalues of just one typically sparse matrix, and to the best of ourknowledge, is by far the fastest available accurate method for estimating the number of communities.  

Several new methods for estimating the number of communities $K$ have been developed concurrently with the present paper. For example, \cite{Riolo&Cantwell&Reinert&Newman2017} use a variant of the Chinese restaurant process to generate community assignments, which automatically yields a choice of $K$; this method is implemented via a Monte Carlo sampling scheme, which is computationally intensive.   A method based on semi-definite programming, another very computationally intensive technique,  was derived and proved to be consistent for assortative networks by \cite{Yan&Sarkar&Cheng2018}.    Improving on \cite{Wang&Bickel2015}, the authors of \cite{Hu&Qin&Yan&Zhao2019} proposed a corrected BIC criterion in \cite{Wang&Bickel2015} to correct for under-estimation.   More recently, \cite{Ma&Su&Zhang2018} combined spectral clustering with binary segmentation to derive a new estimate of $K$. Compared to all these new methods,  the estimators based on Bethe Hessian or non-backtracking matrices we study is still the most computationally efficient, arguably the simplest, and competitive on estimation accuracy (see \cite{Ma&Su&Zhang2018} for some numerical comparisons).    The theoretical analysis of the Bethe-Hessian and the nonbactracking matrices we provide in this paper explain this performance and cover a wider range of settings, including sparse, dense, assortative and disassortative networks; no other method is known to be applicable under a wider range of settings, and most are narrower.

\section{Preliminaries}
Recall $A$  is the $n\times n$ symmetric  network adjacency matrix.  Let $d_i=\sum_{j=1}^n A_{ij}$ be
the degree of node $i$.  Treating $A$ as a  random matrix, let $\E A$ be the expectation of $A$,
and let $d=\frac{1}{n}\sum_{i=1}^n \E d_i$ be the average expected node degree.    

    
\subsection{The non-backtracking matrix}\label{sec: nonbacktracking}

Let $m$ be the number of edges in an undirected network, $2m =  \sum_{i,j=1}^nA_{ij}$.  To construct the  non-backtracking matrix,  we represent the edge between node $i$ and node $j$  by two directed edges,  one from $i$ to $j$ and the other from $j$ to $i$.
The $2m \times 2m$ matrix $\tilde{B}$,  indexed by these directed edges, is defined by
\begin{equation*}
  \tilde{B}_{i \rightarrow j, k \rightarrow l} =
  \left\{
    \begin{array}{ll}
      1 & \hbox{if } j = k \text{ and } i \neq l \\
      0 & \hbox{otherwise.}
    \end{array}
  \right.
\end{equation*}
It is well-known \cite{Angel&Friedman&Hoory2015,Krzakala.et.al2013spectral} that the spectrum of $\tilde{B}$ consists of $\pm 1$ and eigenvalues of an $2n\times 2n$ matrix
\begin{equation}\label{eq: NB def}
  B =
\left(
  \begin{array}{cc}
    0_n & D - I_n \\
    -I_n & A \\
  \end{array}
\right).
\end{equation}
Here $0_n$ is the $n\times n$ matrix of all zeros, $I_n$ is the $n\times n$ identity matrix,
and $D = \mathrm{diag}(d_i)$ is $n \times n$ diagonal matrix with degrees $d_i$ on the diagonal.
It was observed by \cite{Krzakala.et.al2013spectral} that if a network has $K$ communities then
the first $K$ largest (in absolute value) eigenvalues of $B$ are real-valued and well separated from the bulk,
which is contained in a circle of radius $\|B\|^{1/2}$.
We refer to these $K$ eigenvalues as informative eigenvalues of $B$.
It was also shown by \cite{Krzakala.et.al2013spectral} that the spectral norm of the non-backtracking matrix
is approximated by
\begin{equation}\label{eq: non-bactracking spectral norm}
  \tilde{d} = \Big(\sum_{i=1}^n d_i\Big)^{-1} \Big(\sum_{i=1}^n d_i^2\Big) - 1.
\end{equation}

For a special case of a sparse SBM with a bounded expected node degree, \cite{Bordenave.et.al2015non-backtracking} proved
that the leading eigenvalues of $B$ concentrate around non-zero eigenvalues of $\E A$
and the bulk is contained in a circle of radius $\|B\|^{1/2}$, and used the corresponding leading eigenvectors to recover the community labels. The spectrum of $B$ for denser Erdos-Renyi graphs was later analyzed in \cite{Wang&Wood2017}. In particular, if $d\gg n^{5/6}$, then every eigenvalue of $(d-1)^{-1/2}B$ is within a vanishing distance from a limiting spectrum supported on the unit circle of the complex plane. In Theorem~\ref{thm: spectrum nbm rel dense graphs} below we extend this result to much sparser and more general random graphs and require only that $d\gg \log n$.

\subsection{The Bethe Hessian matrix}\label{sec: bethe hessian}

The Bethe Hessian matrix is defined by
\begin{equation}\label{eq: bethe hessian}
  H(r) = (r^2 -1) I - r A + D,
\end{equation}
where $r \in \R$ is a parameter.
In graph theory, the determinant of $H(r)$ is the Ihara-Bass formula for the graph zeta function.
It vanishes if $r$ is an eigenvalue of the non-backtracking matrix \cite{Hashimoto1989,Bass1992,Angel&Friedman&Hoory2015}.   
The Bethe Hessian was used for community detection by \cite{Saade&Krzakala&Zdeborova2014} 
Under the SBM, they argued that the best choice of $r$ is $r_c= \pm \sqrt{d}$,
depending on whether the network is assortative or disassortative;  
for a more general network, they take  $r_c = \pm  \|B\|^{1/2}$.
For assortative sparse networks with $K$ communities and a bounded $d$,
they empirically showed that the $K$ eigenvalues of $H(r_c)$ whose corresponding eigenvectors encode the community structure
are negative, while the bulk of $H(r_c)$ are positive.
Thus, the number of negative eigenvalues of $H(r_c)$ corresponds to the number of communities.
In Theorem~\ref{thm: consistency BH method} below, we prove that this method isindeed consistent for graphs with $d\gg \log n$. 

\section{Spectral estimates of the number of communities}
The spectral properties of the non-backtracking and the Bethe Hessian matrices lead to natural estimates of the number of communities, but they have not been previously considered in this context.   We next outline several spectral methods to determine the number of communities $K$. They are based on simple counts of eigenvalues of either the non-backtracking matrix or the Bethe Hessian matrix, and therefore do not require any adjustment for different models such as SBM or DCSBM.   We list them in Table~\ref{tb: spectral methods}, and proceed to explain the motivation for each one.  

\begin{table}[!ht]
\renewcommand{\arraystretch}{2}
\begin{center}
\begin{tabular}{|l|c|l|}\hline
Method & Parameter & Estimated number of communities $\hat K$ \\
\hline
NB & None  & $\left| \left\{ \l(B)\in\R: \l(B) \ge \|B\|^{1/2} \right\} \right|$ \\
\hline
BHm & $r_m = \left(\frac{\sum_{i=1}^n d_i^2}{\sum_{i=1}^n d_i}  - 1\right)^{1/2}$& $\max\left\{k:\l_{n-k}(H(r_m))\leq 0\right\}$ \\
\hline
BHmc & $r_m = \left(\frac{\sum_{i=1}^n d_i^2}{\sum_{i=1}^n d_i}  - 1\right)^{1/2}$& $\max\{k: t\l_{n-k+1}(H(r_m)) \le \l_{n-k}(H(r_m)) \}$ \\
\hline
BHa &$r_a = \left(\frac{1}{n}\sum_{i=1}^n d_i\right)^{1/2}$ & $\max\{k:\l_{n-k}(H(r_a))\leq 0\}$ \\
\hline
BHac &$r_a = \left(\frac{1}{n}\sum_{i=1}^n d_i\right)^{1/2}$ & $\max\{k: t\l_{n-k+1}(H(r_a)) \le \l_{n-k}(H(r_a)) \}$ \\
\hline
\end{tabular}
\end{center}
\caption{Spectral methods for estimating the number of communities.}
\label{tb: spectral methods}
\end{table}

\subsection{Estimating $K$ from the non-backtracking matrix}\label{sec: NB}

As we will show in Theorems~\ref{thm: NB consistency sparse regime} and \ref{thm: consistency of NBM method dense} under the SBM, the informative eigenvalues of the non-backtracking matrix are real-valued and separated from the bulk of radius $\|B\|^{1/2}$.
Therefore we can estimate $K$ by counting the number of real eigenvalues of $B$ that are at least $\|B\|^{1/2}$. 
We denote this method by NB (for non-backtracking).
As shown by Theorem~\ref{thm: consistency of NBM method dense} and numerical results in Section~5, this estimate of $K$ also works under much more general models with low-rank structure such as DCSBM.
When the network is balanced (communities have similar sizes and edge densities), NB performs well;
however, the accuracy of NB drops if the communities are unbalanced in either size or edge density.  Since $B$ is not symmetric, computing the eigenvalues of $B$ is slightly more demanding than that of the Bethe Hessian matrix for large networks.

\subsection{Estimating $K$ from the Bethe Hessian matrix}\label{sec: BH}

The number of communities corresponds to the number of negative eigenvalues of $H(r)$;  the challenge is in choosing an appropriate value of $r$.   It was argued by \cite{Saade&Krzakala&Zdeborova2014} that when $r=\|B\|^{1/2}$,
the informative eigenvalues of $H(r)$ are negative, while the bulk are positive;
by \cite{Krzakala.et.al2013spectral}, $\|B\|$ can be approximated by $\tilde{d}$ from \eqref{eq: non-bactracking spectral norm}.
Following these results, we first choose $r$ to be $r_m = \tilde{d}^{1/2}$ 
and call the corresponding method BHm.
Simulations show that using $r=r_m$ and $r=\|B\|^{1/2}$ produce similar results;
we choose $r = r_m$ because computing $r_m$ is less demanding than computing $\|B\|^{1/2}$.

Another choice of $r$ is $r_a = \sqrt{(d_1+\cdots+d_n)/n}$, which was proposed by \cite{Saade&Krzakala&Zdeborova2014} for recovering the community structure under the SBM; we call the corresponding method BHa.
We have found that when the network is balanced, NB, BHm and BHa perform similarly;
when the network is unbalanced, BHa produces better results.

Both BHm and BHa tend to underestimate the number of communities, especially when the network is unbalanced.
In that setting, some informative eigenvalues of $H(r)$ become positive, although they may still be far from the bulk.
Based on this observation, we correct BHm and BHa by also using positive eigenvalues of $H(r)$ that are much close to zero than to the bulk.
Namely, we sort eigenvalues of $H(r)$ in non-increasing order $\l_1\ge \l_2\ge\cdots\ge\l_n$, and estimate $K$ by
\begin{equation}\label{eq: K correction}
  \hat{K} = \max\{k: t\l_{n-k+1} \le \l_{n-k} \},
\end{equation}
where $t>0$ is a tuning parameter.
Note that if $\l_{n-k_0+1}<0$ then $\hat{K} \ge k_0$ because $\l_{n-k_0+1} \le \l_{n-k_0}$,
therefore the number of negative eigenvalues of $H(r)$ is always upper bounded by $\hat{K}$.
Heuristically, if the bulk follows the semi-circular law and $\l_{n-k}\ge 0$ is given,
then the probability that $0 \le \l_{n-k+1} \le \l_{n-k}/t$ is less than $1/t$.
When $1/t$ is sufficiently small, we may suspect that $\l_{n-k+1}$ is an informative eigenvalue.
In practice we find that $t\in [4,6]$ works well; we will set $t=5$ for all computations in this paper.
Simulations show that $\hat{K}$ performs well, especially for unbalanced networks.
The resulting methods are denoted by BHmc and BHac, respectively.
We will also use BH to refer to all the methods that use the Bethe Hessian matrix. For a summary of these methods, see Table~\ref{tb: spectral methods}.

\section{Consistency}

The consistency of the non-backtracking matrix based method (NB) for estimating the number of communities in the sparse regime under the stochastic block model follows directly from Theorem~4 of \cite{Bordenave.et.al2015non-backtracking}.
We state this consistency result here for completeness.
The proof given by \cite{Bordenave.et.al2015non-backtracking} is combinatorial in nature and this approach unfortunately does not extend to any other regimes or the Bethe-Hessian matrix.

\begin{theorem}[Consistency in the sparse regime]\label{thm: NB consistency sparse regime}
Consider a stochastic block model with $\pi = (\pi_1,...,\pi_K)$
and $P = (P_{kl})=\frac{1}{n}P^{(0)}$ for some fixed $K\times K$ symmetric matrix $P^{(0)}$.
Assume that 
$(\mathrm{diag}(\pi)P)^r$ has positive entries for some positive integer $r$.
Further, assume that
$E (d_i) = d > 1$ for all $i$,   and all $K$ non-zero eigenvalues of $P$ are greater than $\sqrt{d}$.
Then with probability tending to one as $n\rightarrow\infty$,
the number of real eigenvalues of $B$ that are at least $\|B\|^{1/2}$ is equal to $K$.
\end{theorem}
To better understand the condition on the eigenvalues of $P$, consider the simple model $G(n,\frac{a}{n},\frac{b}{n})$.
This model assumes that there are two communities of equal sizes and nodes are connected with probability $a/n$ if they are in the same community, and $b/n$ otherwise. Since the two non-zero eigenvalues of $P$ are $(a+b)/2$ and $(a-b)/2$, the condition on eigenvalues of $P$ is $(a-b)^2>2(a+b)$.   This matches the phase transition condition for the detectability in the sparse regime \cite{Mossel.et.al.2012,Mossel&Neeman&Sly2014,Massoulie:2014:CDT:2591796.2591857}.

Next, we prove the consistency of the proposed methods  in the denser regime  $d\gg \log n$, sometimes referred to as semi-dense in contrast to the dense regime of $d = O(n)$.  For this regime, we make the following assumptions.

\begin{assumption}\label{as: same node degree}
All nodes have the same expected degree satisfying 
\begin{equation*}
\E \sum_{j=1}^n A_{ij} = d \ge C\log n, \quad 1\le i \le n. 
\end{equation*} 
\end{assumption}
\begin{assumption}\label{as: lowrank EA}
Matrix $\E A$ is of rank $K$ and nonzero eigenvalues of $\E A$ satisfy
\begin{equation*}
|\l_1(\E A)| \ge |\l_2(\E A)| \ge \cdots \ge |\l_K(\E A)| \ge 4d^{1/2} + C(d^{1/4}+(\log n)^{1/2}).
\end{equation*}
\end{assumption}
\begin{assumption}\label{as: upper bound on degree}
The expected degree $d$ in Assumption~\ref{as: same node degree} satisfies
\begin{equation*}
d^5\max_{i,j} \E A_{ij} \le n^{-1/13}.
\end{equation*}
\end{assumption}

Following \cite{Bordenave.et.al2015non-backtracking}, we assume in Assumption~\ref{as: same node degree} that all nodes have the same expected degree. This corresponds perhaps to the most challenging setting where expected degrees alone do not contain information about the latent structure of interest. As in \cite{Bordenave.et.al2015non-backtracking} and \cite{Wang&Wood2017}, this assumption allows us to simplify our analysis of the non-backtracking matrix considerably. If some communities have different expected degrees, we can first use node degrees to identify them and divide the network into sub-networks of similar expected node degrees and apply our results on the sub-networks. 
Note that for the degree-corrected stochastic block model, if the underlying stochastic block model satisfies this assumption and the degree parameters are drawn from the same distribution, then the degree-corrected stochastic block model itself will also satisfy the assumption.  

The lower bound on $\l_K(\E A)$ in Assumption~\ref{as: lowrank EA} is of the form $|\l_K(\E A)|\ge 4(1+o(1))\sqrt{d}$ when $d\gg \log n$. Under $G(n,\frac{a}{n},\frac{b}{n})$, this bound is $(a-b)^2\ge 32(1+o(1))(a+b)$. For a comparison, exact community recovery under $G(n,\frac{a}{n},\frac{b}{n})$ with known number of communities requires $(a-b)^2> 2(a+b+2\sqrt{ab})\log n$ (see e.g. \cite[Theorem~13]{Abbe.com.detect.review2018}).

Assumption~\ref{as: upper bound on degree} guarantees a sharp bound on $\|A-\E A\|$, which is established by \cite{Benaych-georges&Bordenave&Knowles2017}. 
We use this bound in the proofs of Theorem~\ref{thm: consistency of NBM method dense} and Theorem~\ref{thm: consistency BH method} below. For the  Erd\"os-R\'enyi model, Assumption~\ref{as: upper bound on degree} is equivalent to $d\le n^{2/13}$. 
It is unclear if this condition can be removed from the result of  \cite{Benaych-georges&Bordenave&Knowles2017} and consequently from Theorem~\ref{thm: consistency of NBM method dense} and Theorem~\ref{thm: consistency BH method}.    


\begin{theorem}[Consistency of NB based method in the semi-dense regime]\label{thm: consistency of NBM method dense}   
Consider random graphs that satisfy Assumptions~\ref{as: same node degree}, \ref{as: lowrank EA} and \ref{as: upper bound on degree}.
Then with probability at least $1-1/n$, the nonbacktracking matrix has exactly $K$ real eigenvalues with magnitude at least $(1+\varepsilon)\sqrt{d}$ and the remaining eigenvalues are of magnitude smaller than $(1+\varepsilon)\sqrt{d}$, where
$$
\varepsilon = C\left[\left(\frac{\log n}{d}\right)^{1/4}+\left(\frac{1}{d}\right)^{1/8}\right].
$$
\end{theorem}

According to Theorem~\ref{thm: consistency of NBM method dense}, the $K$ informative eigenvalues of the nonbacktracking matrix are separated from the bulk by a circle of radius $(1+\varepsilon)\sqrt{d}$, where $\varepsilon$ is vanishing for $d\gg \log n$. 
Unlike in Theorem~\ref{thm: NB consistency sparse regime}, $K$ is allowed to depend on $n$ in Theorem~\ref{thm: consistency of NBM method dense}.

To compute this estimator in practice, we simply set $\varepsilon=0$ and estimate $d$ with the average observed degree $\bar d = (d_1+\cdots+d_n)/n$.  It is straightforward to show that $\bar d$ is close to $d$ with high  probability.


The key result for proving Theorem~\ref{thm: consistency of NBM method dense} is Theorem~\ref{thm: spectrum nbm rel dense graphs} in Appendix~\ref{app: NBM}, which establishes a connection between the spectra of nonbacktracking and adjacency matrices, and may also be of independent interest. Theorem~\ref{thm: spectrum nbm rel dense graphs} is a significant improvement on Theorem~1.5 in \cite{Wang&Wood2017}, which only considers the Erd\"os-R\'enyi model and requires a much stronger condition,  $d\gg n^{5/6}$ instead of $d\gg \log n$. 


For the Bethe Hessian, no formal results have been previously established. We show in the following theorem that both BHm and BHa methods produce consistent estimator of $K =\mathrm{rank}(\E A) $, provided that the following stronger version of Assumption~\ref{as: lowrank EA} holds.
\begin{assumption}\label{as: modified lowrank EA}
Matrix $\E A$ is of rank $K$ and nonzero eigenvalues of $\E A$ satisfy
\begin{equation*}
\l_1(\E A) \ge \l_2(\E A) \ge \cdots \ge \l_K(\E A) \ge 4d^{1/2} + C(d^{1/4}+(\log n)^{1/2}).
\end{equation*}
\end{assumption}
Note that Assumption~\ref{as: lowrank EA} allows networks to be disassortative, meaning probabilities of connections between communities are higher than within communities, in which case the eigenvalues of $\E A$ may be negative. In contrast, Assumption~\ref{as: modified lowrank EA} requires all eigenvalues of $\E A$ to be non-negative. 

\begin{theorem}[Consistency of the Bethe Hessian matrix method]\label{thm: consistency BH method}
Consider random graphs that satisfy Assumptions~\ref{as: same node degree}, \ref{as: upper bound on degree} and \ref{as: modified lowrank EA}.
Then with probability at least $1-1/n$, the Bethe Hessian $H(r)$ with $r=(1+\varepsilon)r_m$ or $r=(1+\varepsilon)r_a$ and $\varepsilon = C\sqrt{\log n/d}$ has exactly $K$ negative eigenvalues.
\end{theorem}

Again in practice, we set $\varepsilon=0$ to compute the estimator.  

\section{Numerical results}

In this section, we briefly compare the empirical accuracy of estimating the number of communities by using the non-backtracking matrix (NB), and all the versions based on the Bethe Hessian matrix  (BHm, BHmc, BHa, and BHac), described in Section~\ref{sec: NB} and Section~\ref{sec: BH}. 
We compare them with two other methods representative of approaches in the literature to estimating the number of communities in networks: the network cross-validation method (NCV) proposed by \cite{Chen&Lei2014} and
a likelihood-based BIC-type method (VLH, for variational likelihood) proposed by \cite{Wang&Bickel2015}.
We use NCVbm and NCVdc to denote the versions of the NCV method specifically designed for the SBM and the DCSBM, respectively;  VLH is only designed to work under the SBM, so it is not included in the DCSBM comparisons.
To make comparisons with VLH computationally feasible, instead of using the variational method to estimate the posterior of the community labels
as done by \cite{Wang&Bickel2015}, we first estimate the node labels by the pseudo-likelihood method proposed by \cite{amini2013pseudo} and then compute the posterior following  \cite{Wang&Bickel2015}.  In small-scale simulations where both approaches are computationally feasible (results omitted) we found that substituting pseudo-likelihood for the variational method has very little effect on the estimate of $K$.    The tuning parameter of VLH is set to one following \cite{Wang&Bickel2015}.   We do not include the method of \cite{Bickel&Sarkar2013} in these comparisons due to its high computational cost.    Note that our theoretical analysis assumes for simplicity that all expected node degrees are equal  (Theorems~\ref{thm: NB consistency sparse regime}, \ref{thm: consistency of NBM method dense} and \ref{thm: consistency BH method});  however, we allow different expected node degrees in simulations.   In this section,  $d=\frac{1}{n}\sum_{i=1}^n \E d_i$ denotes the average expected node degree.   

\subsection{Synthetic networks}

To generate synthetic networks, we fix the labels $c \in \{1,...,K\}^n$ so that $c_i = k$ if $n\pi_{k-1} + 1 \le i < n\pi_k$, where $\pi_0=0$.
The label matrix $Z \in \R^{n\times K}$, given by  $Z_{ik}={\mathbf 1}(c_i=k)$,  encodes $c$ by representing each node's label with a row of $K$ elements, exactly one of which is equal to 1,  and the rest are equal to 0.  
Let $\tilde{P}$ be a $K \times K$ matrix with the diagonal $w=(w_1,...,w_K)$ and off-diagonal entries $\beta$, and $M = ZPZ^T$.
Under the stochastic block model, we generate entires of $A$ using the edge probability matrix $E (A) = \rho_n M$;   the average degree $d$ is controlled by $\rho_n$.
The parameter $w$ controls the relative edge densities within communities, and $\beta$ controls the out-in probability ratio. Smaller values of $\beta$ and larger values of $d$ make the problem easier.
For the DCSBM, we generate the degree parameters $\theta_i$ from a distribution that takes two values,  $\P(\theta=1)=1-\gamma$ and $\P(\theta=0.2)=\gamma$.
Parameter $\gamma$ controls the fraction of ``hubs'', the high-degree nodes allowed under the DCSBM, and setting $\gamma = 0$ gives back the regular SBM.
Given $\theta=(\theta_i,...,\theta_n)$, the edges are generated independently with
probabilities $E (A) = \rho_n \mathrm{diag}(\theta)M\mathrm{diag}(\theta)$,
where $\mathrm{diag}(\theta)$ is a diagonal matrix with $\theta_i$'s on the diagonal.

The number of nodes is set to $n=1200$,
the out-in probability ratio $\beta = 0.2$,
and we vary the average degree $d$,  weights $w$, and community sizes determined by the vector $\pi$.   We consider three different values for the number of communities, $K = 2$, 4, and 6.
For each setting, we generate $200$ replications of the network and record the accuracy, defined as
the fraction of times a method correctly estimates the true number of communities $K$.   The methods NCV and VLH require a pre-specified set of $K$ values to choose from;  we use the set $\{1,2,...,8\}$ for synthetic networks
and $\{1,2,...,15\}$ for real-world networks.

We start by varying the average degree $d$, which controls the overall difficulty of the problem, while keeping community sizes equal.    Figure~\ref{fig:sparsity}  shows the performance of all methods for the balanced community density case, $w_i=1$ for all $1\le i \le K$.   Figure~\ref{fig:unbalanced w} shows the unbalanced case, with $w=(1,2)$ for $K=2$, $w=(1,1,2,3)$ for $K=4$, and $w=(1,1,1,1,2,3)$ for $K=6$.   In every figure, the top row corresponds to the SBM ($\gamma = 0$) and the bottom row to the DCSBM ($\gamma=0.9$, meaning 10\% of nodes are hubs).

In general, we see that when everything is balanced (Figure~\ref{fig:sparsity}), all spectral methods perform fairly similarly and outperform both cross-validation (NCV) and the BIC-type criterion (VLH).  Also, for larger $K$ and especially under DCSBM, the corrected versions are somewhat better than the uncorrected ones, and the best Bethe Hessian methods are better than the non-backtracking estimator.

For networks with equal size communities but different edge densities within communities (Figure~\ref{fig:unbalanced w}), cross-validation performs poorly, but VLH relatively improves.   For larger $K$ the spectral methods are also distinguishable, with all BH methods dominating NB, and corrected versions providing improvement.  Overall, BHac is the best spectral method,  with VLH comparable for the SBM in this case.  The BHac method is the best overall for DCSBM where VLH is not applicable.
\begin{figure}[!ht]
  \centering
  \includegraphics[trim=110 10 100 10,clip,width=1\textwidth]{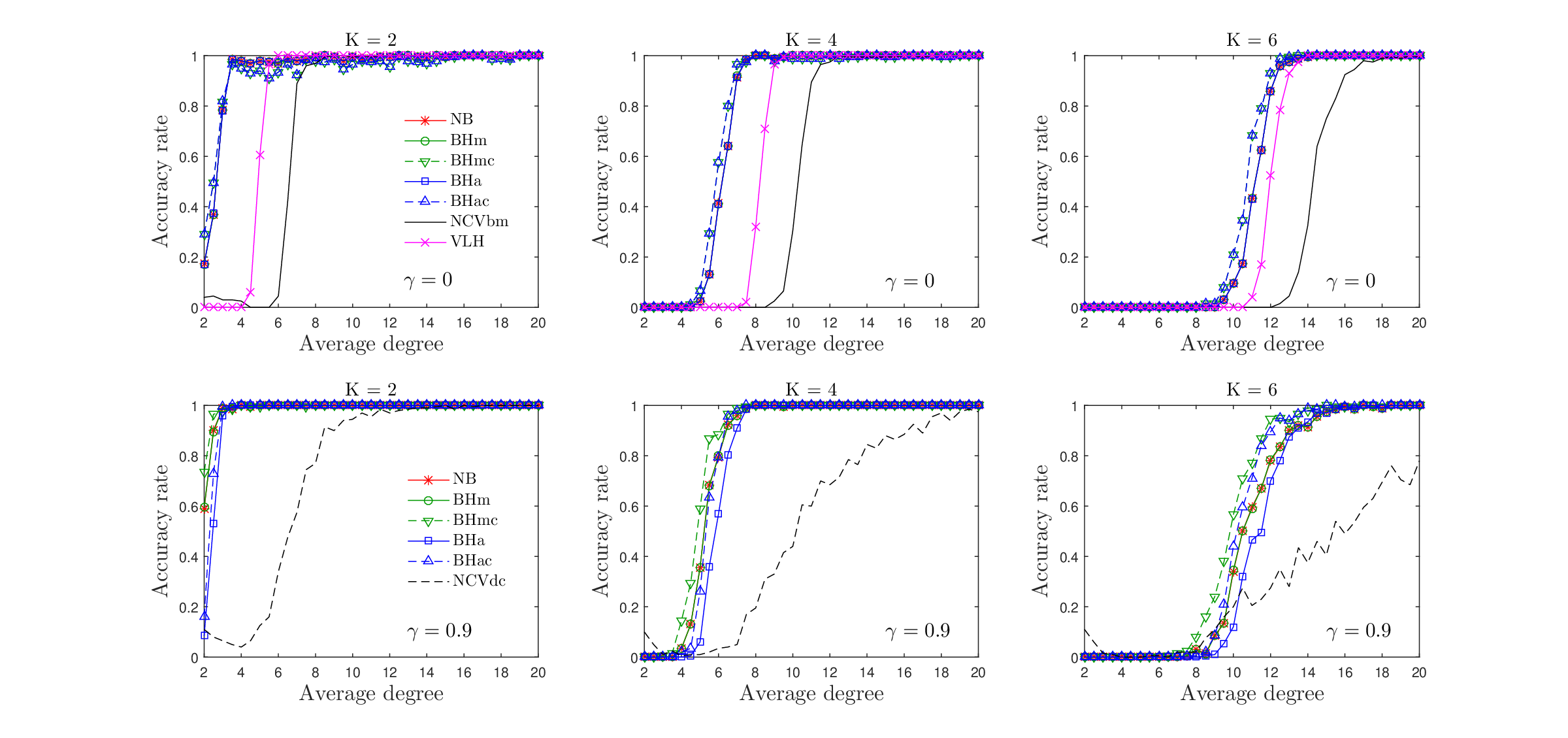}\\

  \doublespacing
  \caption{The accuracy of estimating $K$  as a function of the average degree.
    All communities have equal sizes, and $w_i=1$ for all $1\le i \le K$. }
  \label{fig:sparsity}
\end{figure}

\begin{figure}[!ht]
  \centering
  \includegraphics[trim=110 10 100 10,clip,width=1\textwidth]{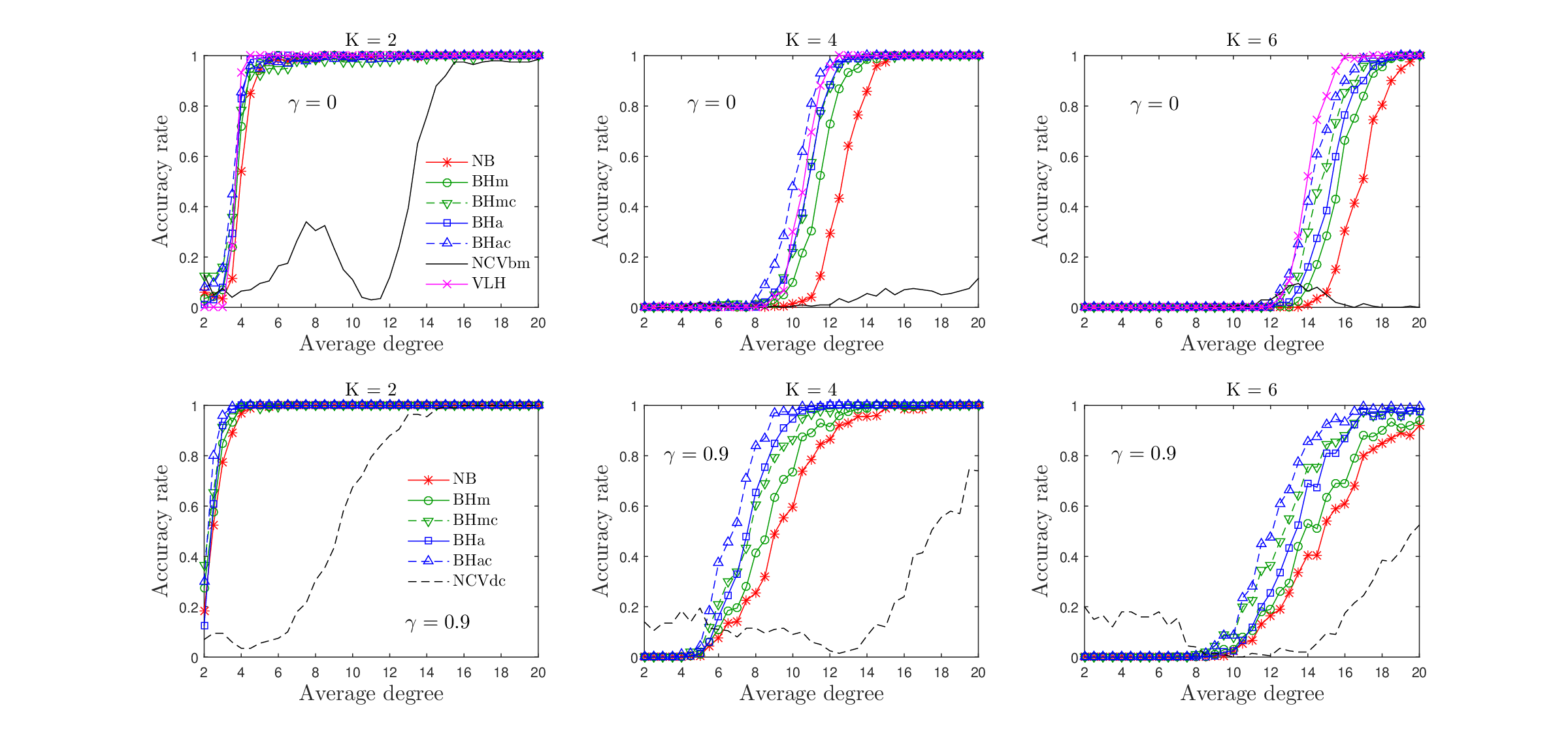}\\

  \doublespacing
  \caption{The accuracy of estimating $K$ as a function of the average degree.
  All communities have equal sizes; $w=(1,2)$ for $K=2$, $w=(1,1,2,3)$ for $K=4$, and $w=(1,1,1,1,2,3)$ for $K=6$.}
  \label{fig:unbalanced w}
\end{figure}

Communities of different sizes present a challenge for community detection methods in general, and the presence of relatively small communities makes the problem of estimating $K$ difficult.  To test the sensitivity of all the methods to this factor, we change the proportions of nodes falling into each community
setting $\pi_1 = r/K$, $\pi_K = (2-r)/K$, and $\pi_i = 1/K$ for $2 \le i \le K-1$, and varying $r$ in the range $[0.2, 1]$.
As $r$ increases, the community sizes become more similar, and are all equal when $r=1$.
Figure \ref{fig:community_ratio} shows the performance of all methods as a function of $r$.
The top row corresponds to the SBM ($\gamma = 0$), the bottom row to the DCSBM ($\gamma=0.9$), and the within-community edge density parameters
$w_i = 1$ for all $1 \le i \le K$.  Here we see that VLH is less sensitive to $r$ than the spectral methods, but unfortunately it is not available under the DCSBM.   Cross-validation is still dominated by spectral methods except for very small values of $r$, where all methods perform poorly.    The corrections still provide a slight improvement for Bethe Hessian based methods, although all spectral methods perform fairly similarly in this case.

\begin{figure}[!ht]
  \centering
  \includegraphics[trim=110 10 100 10,clip,width=1\textwidth]{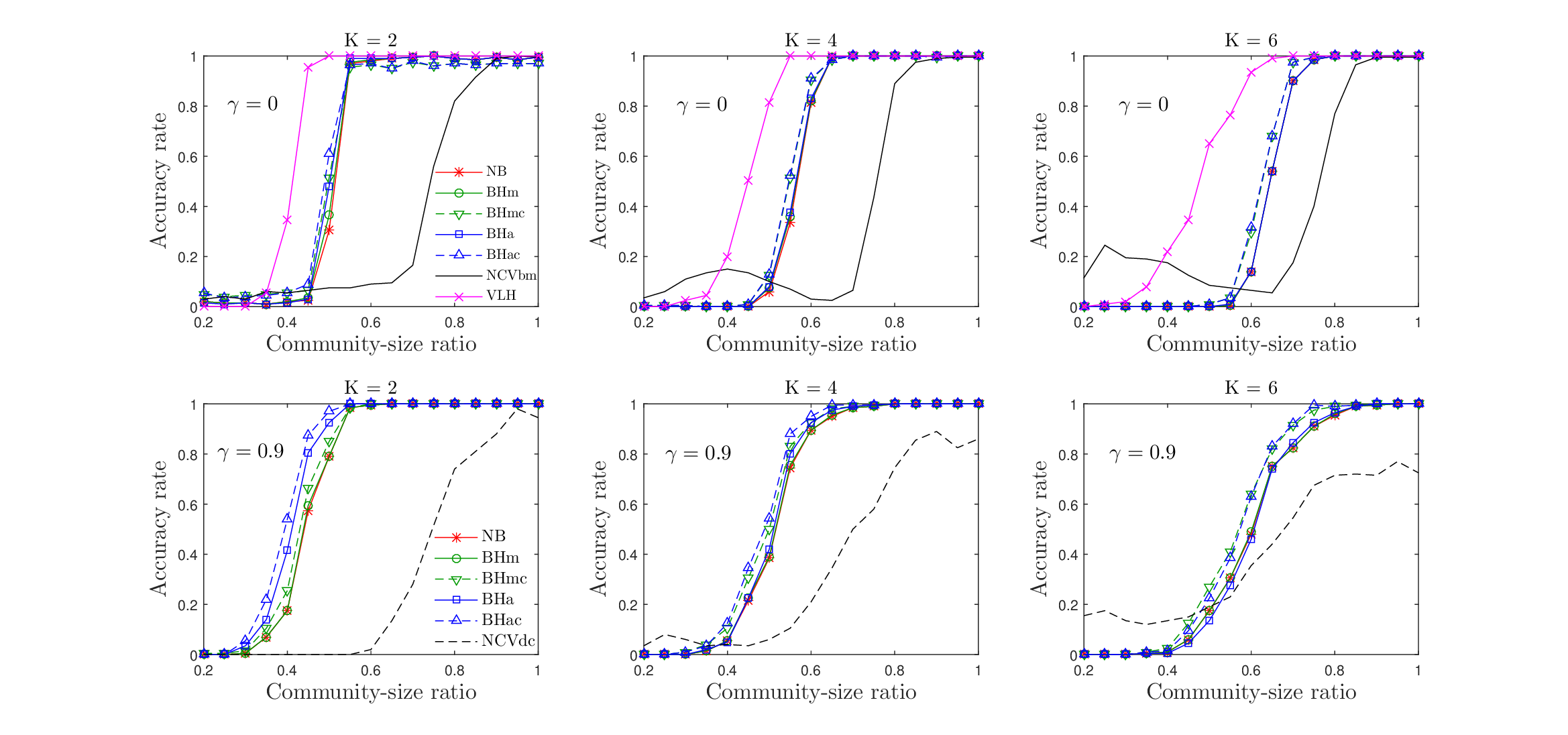}\\

  \doublespacing
  \caption{The accuracy of estimating  $K$ as a function of the community-size ratio  $r$:
  $\pi_1 = r/K$, $\pi_K = (2-r)/K$, and $\pi_i=1/K$ for $2 \le i \le K-1$. In all plots, $w_i=1$ for $1\le i \le K$;
  the average degrees are $\l_n = 10$ (left), $15$ (middle), and $20$ (right).}
  \label{fig:community_ratio}
\end{figure}

%

\subsection{Real world networks}

Finally, we apply the proposed methods on several popular network datasets which come with the ``ground truth'' node labels and the corresponding number of communities.   We note that the network structure itself can indicate a different number of communities than those given in the ground truth, since those are typically derived from one specific node attribute and there may be other communities or sub-communities corresponding to different attributes.   However, these ground truth labels still provide a reasonable baseline against which to compare estimators.   

The college football network \cite{Girvan&Newman2002} represents 115 US college football teams and the games they played in 2000.   The ``ground truth'' communities are the 12 conferences that the teams belong to.
The political books network \cite{Newman2006}, compiled around 2004,  consists of 105 books about US politics;  an edge is ``frequently purchased together'' on Amazon.  The $K=3$ communities are ``conservative'', ``liberal'', or ``neutral'', labelled manually based on contents.
The dolphin network \cite{Lusseau2003} is a social network of 62 dolphins, with
edges representing social interactions, and $K=2$ communities are based on a split which happened after one dolphin left the group.
Similarly, the karate club network  \cite{Zachary1977} is a social network of 34 members of a karate club,  with
edges representing friendships, and $K=2$ communities based on a split following a dispute.   Finally, the political blogs network \cite{Adamic05}, collected around 2004,
 consists of blogs about US politics, with edges representing web links, and $K=2$ communities are ``conservative'' and ``liberal'', based on manual labelling.  For this dataset, as is commonly done in the literature, we only consider its largest connected component of 1222 nodes.

Table~\ref{Table:real data} shows the estimated number of communities in these networks.
All spectral methods estimate the correct number of communities for dolphins and the karate club, and do a reasonable job for the college football and political books data.   For political blogs, all methods but NCV and VLH estimate a much larger number of communities, suggesting the estimates correspond to smaller sub-communities with more uniform degree distributions that have been previously detected by other authors.   We also found that the VLH method was highly dependent on the tuning parameter, and the estimates by NCVbm and NCVdc varied noticeably from run to run due to their use of random partitions.
\begin{table}[!ht]
\renewcommand{\arraystretch}{1.2}
\centering
\resizebox{\linewidth}{!}{%
\begin{tabular}{l|ccccccccc}
  \hline
   \textbf{Dataset} & NB & BHm  & BHmc & BHa & BHac & NCVbm & NCVdc & VLH & Truth\\ \hline
  \text{College football } & 10 & 10 & 10 & 10 & 10 & 14 & 13 & 9 & 12\\
  \text{Political books }   & 3 & 3 & 4 & 4 & 4 & 8 & 2 & 6 & 3\\
  \text{Dolphins}   & 2 & 2 & 2 & 2 & 2 & 4 & 3 & 2 & 2\\
  \text{Karate club}   & 2 & 2 & 2 & 2 & 2 & 3 & 3 & 4 & 2\\
  \text{Political blogs}   & 8 & 7 & 8 & 7 & 8 & 10 & 2 & 1 & 2\\ \hline
\end{tabular}}
\caption{Estimates of the number of communities in real-world networks.}
\label{Table:real data}
\end{table}

\section{Discussion}
The numerical experiments suggest that the spectral methods provide extremely fast and reliable estimates of the number of communities $K$ for balanced networks, with the Bethe Hessian based method with the threshold choice $r_a$ and the correction described in \eqref{eq: K correction} the best choice in most scenarios.   With communities of  significantly different sizes, they tend to underestimate $K$  by combining small communities together, which seems to be an intrinsic limitation of spectral methods.
This suggests that their estimates can be used as a lower bound on $K$ and a starting point for a  more elaborate and computationally demanding likelihood-based method like VLH, in the same way that spectral clustering can be used to initialize a more sophisticated community detection method.   Having a small set of plausible values of $K$ to focus on can significantly reduce the computational cost and improve the accuracy of estimating the number of communities.

For semi-dense networks, we show in Theorems~\ref{thm: consistency of NBM method dense} and \ref{thm: consistency BH method} that estimating the number of communities is possible below the exact recovery threshold. For example, under $G(n,\frac{a}{n},\frac{b}{n})$, our results require $(a-b)^2\ge 32(1+o(1))(a+b)$ while exact community recovery is feasible if $(a-b)^2> 2(a+b+2\sqrt{ab})\log n$. Determining the exact condition under which estimating the number of communities is possible is an interesting and challenging question and we leave it for future research. 

\appendix

\section{Proof of Theorem~\ref{thm: consistency of NBM method dense}}\label{app: NBM}

Following \cite{Wang&Wood2017}, we will work with the following rescaled conjugation of the nonbacktracking matrix $B$ defined in \eqref{eq: NB def} (which has the same eigenvalues as $B/\sqrt{\alpha}$ where $\alpha = d-1$)
\begin{equation}\label{eq: scaled NBM}
\begin{pmatrix}
\frac{1}{\sqrt{\alpha}}A & \frac{1}{\sqrt{\alpha}}(I - D) \\
I & 0
\end{pmatrix}
= 
\begin{pmatrix}
\frac{1}{\sqrt{\alpha}}A & -I \\
I & 0
\end{pmatrix} +
\begin{pmatrix}
0 & \frac{1}{\alpha}(\E D - D) \\
0 & 0
\end{pmatrix}
=: H + E.
\end{equation}
The key result for proving Theorem~\ref{thm: consistency of NBM method dense} is Theorem~\ref{thm: spectrum nbm rel dense graphs} below, which establishes a connection between spectra of $H+E$ and $H$. The spectrum of $H$ is closely related to the spectrum of the adjacency matrix, and is discussed in Section~\ref{sec: spectrum of H}.

To prove Theorem~\ref{thm: spectrum nbm rel dense graphs}, we only
need a crude bound on $\|A-\E A\|$ that is known to hold for very
general graph models,  including SBM, DCSBM and inhomogeneous
Erdos-Renyi models \cite{le2017concentration}. For clarity, we put this bound in Assumption~\ref{as: concentration of adjacency matrix} below. We will replace it with a sharper bound in Theorem~\ref{thm: matrix concentration} to prove Theorem~\ref{thm: consistency of NBM method dense}.
\begin{assumption}\label{as: concentration of adjacency matrix}
With probability at least $1-1/n$, the following inequality holds
\begin{equation*}
\|A-\E A\| \le C\sqrt{d}.
\end{equation*} 
\end{assumption}
It is easy to see that Assumption~\ref{as: concentration of adjacency matrix} implies $\|E\| = O(1/\sqrt{d})$ with high probability while \cite{Wang&Wood2017} shows that $H$ is diagonalizable as follows.

\subsection{Spectrum of $H$}\label{sec: spectrum of H}
Denote by $v_1,...,v_n$ and $\l_1,\l_2,...,\l_n$ eigenvectors and corresponding eigenvalues of $A/\sqrt{\alpha}$ ordered so that $|\l_1|\ge |\l_2| \ge...\ge|\l_n|$. For each $i$, $H$ has two eigenvalues $\mu_{2i-1}$ and $\mu_{2i}$ that are solutions of equation $\mu^2 - \lambda_i \mu +1 = 0$, that is 
\begin{equation}\label{eq: mu formula}
\mu_{2i-1} = \frac{\l_i + \sqrt{\l_i^2 - 4}}{2}, \quad
\mu_{2i} = \frac{\l_i - \sqrt{\l_i^2 - 4}}{2}.
\end{equation}
The corresponding left (unit) eigenvectors of $H$ are
$$
y_{2i-1}^* = \frac{1}{\sqrt{1+|\mu_{2i-1}|^2}}(-\mu_{2i-1}v_i^T, v_i^T), \quad
y_{2i}^* = \frac{1}{\sqrt{1+|\mu_{2i}|^2}}(-\mu_{2i}v_i^T, v_i^T)
$$
and their inner product is
\begin{equation}\label{eq: y inner product}
\ip{y_{2i-1}}{y_{2i}} = 
\begin{cases}
\frac{\l_i^2+\l_i\sqrt{\l_i^2-4}}{4},&\text{if }
|\l_i|<2 \\
\frac{2}{|\l_i|}, & \text{if }
|\l_i|\ge 2
\end{cases} \ = \
\begin{cases}
\frac{\l_i\mu_{2i-1}}{2}, & \text{if }
|\l_i|<2 \\
\frac{2}{|\l_i|}, & \text{if }
|\l_i|\ge 2.
\end{cases}
\end{equation}
The corresponding right eigenvectors of $H$ are proportional to 
\begin{equation}\label{eq: x formula}
x_{2i-1} = \frac{\sqrt{1+|\mu_{2i-1}|^2}}{\mu_{2i}-\mu_{2i-1}} 
\begin{pmatrix}
v_i\\
\mu_{2i}v_i
\end{pmatrix}, \quad
x_{2i} = \frac{\sqrt{1+|\mu_{2i}|^2}}{\mu_{2i-1}-\mu_{2i}} 
\begin{pmatrix}
v_i\\
\mu_{2i-1}v_i
\end{pmatrix},
\end{equation}
with inner product
\begin{equation}\label{eq: x inner product}
\ip{x_{2i-1}}{x_{2i}} = 
\begin{cases}
\frac{\l_i^2+\l_i\sqrt{\l_i^2-4}}{\l_i^2-4},&\text{if }
|\l_i|<2 \\
\frac{2|\l_i|}{4-\l_i^2}, & \text{if }
|\l_i|\ge 2
\end{cases} \ = \
\begin{cases}
\frac{2\l_i\mu_{2i-1}}{\l_i^2-4}, & \text{if }
|\l_i|<2 \\
\frac{2|\l_i|}{4-\l_i^2}, & \text{if }
|\l_i|\ge 2
\end{cases}
\end{equation}
Note that $x_{2i-1}$ and $x_{2i}$ are not unit vectors. Their squared norms are
\begin{equation}\label{eq: norm of x}
\|x_{2i-1}\|^2 = \|x_{2i}\|^2 = 
\begin{cases}
\frac{4}{4-\l_i^2},&\text{if }
|\l_i|<2 \\
\frac{\l_i^2}{\l_i^2 - 4}, & \text{if }
|\l_i|\ge 2.
\end{cases}
\end{equation}
It is convenient to not normalize $x_{2i-1}$ and $x_{2i}$ because $H$ admits the decomposition
$$
H = \sum_{i=1}^n \left(\mu_{2i-1}x_{2i-1}y^*_{2i-1}+\mu_{2i}x_{2i}y^*_{2i}\right).
$$
Note that from the formulas above we have
$$
x_{2i-1}\perp y_{2i}, \quad x_{2i}\perp y_{2i-1}, \quad \ip{x_{2i-1}}{y_{2i-1}} = \ip{x_{2i}}{y_{2i}} = 1. 
$$
The space $\mathbb{C}^{2n}$ can be decomposed as a direct sum of orthogonal two-dimensional subspaces $\text{span}\{x_{2i-1},x_{2i}\}=\text{span}\{y_{2i-1},y_{2i}\}$, which are invariant under the action of $H$. Moreover, the orthogonal projection onto $\text{span}\{x_{2i-1},x_{2i}\}$ is given by $x_{2i-1}y^*_{2i-1}+x_{2i}y^*_{2i}$.

\subsection{Spectrum of $H+E$} 
The main difficulty of analyzing the spectrum of $H+E$ is that $H$ and $E$ are not symmetric so standard Weyl's inequalities do not apply even though $\|E\|$ is small. 
%
%
Wang and Wood \cite{Wang&Wood2017} use the Bauer-Fike theorem instead
and show that for Erdos-Renyi random graphs, the perturbation of $E$
is negligible if the average degree is at least of order
$n^{5/6}$. This strong assumption is likely an artifact of their proof
because the Bauer-Fike bound is often not tight. In fact, by a direct and more careful analysis we show in the following  theorem that the spectrum of $H+E$ is close to the spectrum of $H$ for much sparser graphs.

\begin{theorem}[Connection between spectra of non-backtracking and adjacency matrices]\label{thm: spectrum nbm rel dense graphs}
There exists a constant $C>0$ such that the following holds. Consider random graphs satisfying Assumptions~\ref{as: same node degree} and \ref{as: concentration of adjacency matrix}. Then with probability at least $1-1/n$, for each eigenvalue $\beta$ of $H+E$, there exists an eigenvalue $\mu$ of $H$ such that 
$$|\beta - \mu| \le C d^{-1/8}.$$
\end{theorem}

For proving Theorem~\ref{thm: consistency of NBM method dense}, we replace Assumption~\ref{as: concentration of adjacency matrix} with the following shaper bound on $\|A-\E A\|$, which holds under stronger assumptions. This bound follows directly from \cite{Benaych-georges&Bordenave&Knowles2017} and \cite{VanVuRanDiscMat2008}; see also \cite{Wang&Wood2017}.
\begin{theorem}[Concentration of adjacency matrix]\label{thm: matrix concentration}
There exists a constant $C_1,C_2>0$ such that the following holds. Assume that 
$$
d\ge C_1\log n \quad \text{and} \quad d^5\max_{i,j} \E A_{ij} \le n^{-1/13}.
$$ 
Then with probability at least $1-1/n$, we have
$$
\|A-\E A\| \le 2\sqrt{d} + C_2\sqrt{\log n}.
$$
\end{theorem} 

We are now ready to prove Theorem~\ref{thm: consistency of NBM method dense}.

\begin{proof}[Proof of Theorem~\ref{thm: consistency of NBM method dense}]
Let $\l_1(\E A),...,\l_K(\E A)$ be the nonzero eigenvalues of $\E A$ and $\l_1(A),\cdots,\l_n(A)$ be eigenvalues of $A$, ordered so that $|\l_1(\E A)|\ge\cdots\ge|\l_K(\E A)|>0$ and $|\l_1(A)|\ge\cdots\ge|\l_n(A)|$. Then by Weyl's inequality and Theorem~\ref{thm: matrix concentration}, with probability at least $1-1/n$ we have
\begin{eqnarray*}
|\l_i(A)| &\le& 2\sqrt{d} +C\sqrt{\log n} \quad \mathrm{for} \quad i\ge K+1,  \\
|\l_i(A) - \l_i(\E A)| &\le& 2\sqrt{d} +C\sqrt{\log n} \quad \mathrm{for} \quad 1\le i \le K.
\end{eqnarray*} 
Since $|\l_K(\E A)|\ge 4\sqrt{d}+4C(\sqrt[4]{d}+\sqrt{\log n})$ by Assumption~\ref{as: lowrank EA}, it follows that $|\l_i(A)|\ge  2\sqrt{d} +2C(\sqrt[4]{d}+\sqrt{\log n})$ for $1\le i\le K$. Therefore for $1\le i\le K$, from \eqref{eq: mu formula} we have
\begin{eqnarray*}
\max\{|\mu_{2i-1}(H)|,|\mu_{2i}(H)|\} &\ge& 1+\frac{2C(\sqrt[4]{d}+\sqrt{\log n})^{1/2}}{d^{1/4}}>1,\\
\min\{|\mu_{2i-1}(H)|,|\mu_{2i}(H)|\} &=& \frac{1}{\max\{|\mu_{2i-1}(H)|,|\mu_{2i}(H)|\}} <1.
\end{eqnarray*}
Similarly, for $i\ge K+1$ we have
$$
\max\{|\mu_{2i-1}(H)|,|\mu_{2i}(H)|\} < 1+2C\left(\frac{\log n}{d}\right)^{1/4}.
$$
Theorem~\ref{thm: spectrum nbm rel dense graphs} and the continuity of eigenvalues with respect to small perturbation then imply that for $1\le i\le K$,
\begin{eqnarray*}
\max\{|\mu_{2i-1}(H+E)|,|\mu_{2i}(H+E)|\} &\ge& 1+ \frac{2C(\sqrt[4]{d}+\sqrt{\log n})^{1/2}}{d^{1/4}} - C d^{-1/8} \\ 
&\ge& 1+2C\left(\frac{\log n}{d}\right)^{1/4}+Cd^{-1/8},
\end{eqnarray*}
while the remaining eigenvalues of $H+E$ have magnitude at most 
$$
1+2C\left(\frac{\log n}{d}\right)^{1/4}+Cd^{-1/8}.
$$
Since $B = \sqrt{\alpha}(H+E)$ by \eqref{eq: NB def} and \eqref{eq: scaled NBM}, it follows that the nonbacktracking matrix has exactly $K$ eigenvalues with magnitude at least $(1+\varepsilon)\sqrt{d}$ and the remaining eigenvalues are of magnitude smaller than $(1+\varepsilon)\sqrt{d}$.

To show that the $K$ largest eigenvalues in magnitude of $B$ are real, we use the following deterministic inclusion bound for the spectrum of $B$; see \cite[Theorem 3.7]{Angel&Friedman&Hoory2015}. Let $d_{\min}\ge 2$ and $d_{\max}$ be the minimal and maximal degrees of a graph. Then the spectrum of $B$ satisfies
$$
\sigma(B)\subseteq\left\{\lambda\in\C:\sqrt{d_{\min}-1}\le |\lambda| \le\sqrt{d_{\max}-1}
\right\}\cap\left\{\lambda\in\R:1\le|\lambda|\le d_{\max}-1\right\}.
$$ 
In our setting, we bound $d_{\max}$ using standard Bernstein's inequality: with probability at least $1-1/n$, 
$$
\sqrt{d_{\max}-1} \le \sqrt{d + C\sqrt{d\log n}} \le (1+\varepsilon)\sqrt{d}.
$$
Since all complex eigenvalues of $B$ are contained in a circle of radius at most $\sqrt{d_{\max}-1}$, the $K$ largest eigenvalues of $B$ in magnitude, which are outside the circle of radius $(1+\varepsilon)\sqrt{d}$, must be real. The proof is complete. 
\end{proof}

The rest of this section is devoted to proving Theorem~\ref{thm: spectrum nbm rel dense graphs}. Besides the facts listed in Section~\ref{sec: spectrum of H}, we need the following elementary lemmas, the proofs of which are postponed until the end of this section.

\begin{lemma}\label{lem: nonsingularity bound}
Let $x,y,v$ be unit vectors with $|\ip{x}{y}|\le 1-\varepsilon$ for some $\varepsilon\in[0,1]$, $v\in\text{span}\{x,y\}$ and $a,b\in \mathbb{C}$ be any complex numbers. Then
$$
\|ax+by\|^2 \ge \varepsilon(|a|^2+|b|^2), \quad
|\ip{v}{x}|^2 + |\ip{v}{y}|^2 \ge \varepsilon.
$$
\end{lemma}

\begin{lemma}\label{lem: span x}
Let $x_{2i-1},x_{2i}$ be right eigenvectors of $H$ given by \eqref{eq: x formula}.  Then for any $a,b\in\mathbb{C}$ and $1\le i \le n$ we have
$$
\|ax_{2i-1}+bx_{2i}\| \ge \max\{|a|,|b|\}.
$$
\end{lemma}

\begin{lemma}\label{lem: H is bounded}
Let  $x_{2i-1},x_{2i}$ be right eigenvectors of $H$ given by \eqref{eq: x formula} and denote $W_i = \text{span}\{x_{2i-1},x_{2i}\}$. Then for any $1\le i \le n$ we have
$$
\sup_{w\in W_i} \|Hw\| \le 4\max\{|\l_i|,1\}\cdot\|w\|.  
$$
\end{lemma}

We are now ready to prove Theorem~\ref{thm: spectrum nbm rel dense graphs}.

\begin{proof}[Proof of Theorem~\ref{thm: spectrum nbm rel dense graphs}]
Denote by $P_i$ the orthogonal projection onto $\text{span}\{x_{2i-1},x_{2i}\}$. Let $u$ be a unit eigenvector of $H+E$ with corresponding eigenvalue $\beta$ and $u_i = P_iu/\|P_iu\|$. 
Note first that
$$
u = \sum_i P_iu = \sum_i (x_{2i-1}y_{2i-1}^* + x_{2i}y_{2i}^*) P_i u. 
$$   
This allows us to write $Eu$ as follows:
$$
Eu = \beta u - Hu = \sum_i \left[(\beta - \mu_{2i-1})x_{2i-1}y_{2i-1}^* + (\beta -\mu_{2i})x_{2i}y_{2i}^*\right] P_i u.
$$
Note that the terms in above sum belong to orthogonal subspaces of $\mathbb{C}^{2n}$. Therefore 
\begin{equation}\label{eq: master bound}
\|E\|^2 \ge \sum_i \left\| \left[(\beta-\mu_{2i-1})x_{2i-1}y_{2i-1}^* + (\beta-\mu_{2i})x_{2i}y_{2i}^*\right] u_i \right\|^2 \|P_i u\|^2
= \sum_i T_i  \|P_i u\|^2
\end{equation}
where $T_i$ denotes the first factor of the corresponding term in the sum. 

\medskip
Let $\varepsilon\in(0,1/4)$ be a small number to be chosen later. Consider first the eigenvalues $\l_i$ with magnitude not close to $2$, namely those satisfying $||\l_i|-2|>\varepsilon$. From \eqref{eq: x inner product} and \eqref{eq: norm of x} we have
\begin{equation}\label{eq: x and y inner products}
|\ip{y_{2i-1}}{y_{2i}}| = 
\frac{|\ip{x_{2i-1}}{x_{2i}}|}{\|x_{2i-1}\|\cdot \|x_{2i}\|} \ = \ 
\begin{cases}
|\l_i|/2, & \text{if } |\l_i| < 2- \varepsilon \\
2/|\l_i|, & \text{if } |\l_i| > 2+ \varepsilon
\end{cases}
\ \le \ 1 - \varepsilon/3.
\end{equation}
It also follows from \eqref{eq: norm of x} that $\|x_{2i-1}\| = \|x_{2i}\|>1$. Since $u_i\in \text{span}\{x_{2i-1},x_{2i}\}=\text{span}\{y_{2i-1},y_{2i}\}$, if $||\l_i|-2|>\varepsilon$ then by \eqref{eq: x and y inner products} and Lemma~\ref{lem: nonsingularity bound} (applied to $\|x_{2i-1}\|^{-1}x_{2i-1}$, $\|x_{2i}\|^{-1}x_{2i}$ first and then to $y_{2i-1},y_{2i}$) we have
\begin{eqnarray}
\nonumber T_i &\ge& \varepsilon/3\cdot\left(|\beta - \mu_{2i-1}|^2 |y_{2i-1}^*u_i|^2 + |\beta - \mu_{2i}|^2 |y_{2i}^*u_i|^2\right)\cdot\|x_{2i}\|^2\\
\label{eq: Ti bound}&\ge& \varepsilon^2/9\cdot\min\{|\beta-\mu_{2i-1}|^2,|\beta-\mu_{2i}|^2\}.  
\end{eqnarray}
We now consider two cases of $u$, namely whether the following inequality holds: 
\begin{equation}\label{eq: two cases of u}
\sum_{||\l_i|-2|>\varepsilon} \|P_i u\|^2 > \varepsilon.
\end{equation} 
We will show that in both cases there exists an eigenvalue of $H$ that is close  to $\beta$. Assume first that \eqref{eq: two cases of u} holds. Then from 
\eqref{eq: master bound}, \eqref{eq: Ti bound} and \eqref{eq: two cases of u} we have
\begin{eqnarray*}
\|E\|^2 &\ge& \sum_{||\l_i|-2|>\varepsilon} T_i \cdot \|P_i u\|^2 \\
&\ge& \sum_{||\l_i|-2|>\varepsilon} \varepsilon^2/9\cdot\min\{|\beta-\mu_{2i-1}|^2,|\beta-\mu_{2i}|^2\}\cdot \|P_i u\|^2 \\
&\ge& \varepsilon^2/9\cdot\min_{||\l_i|-2|>\varepsilon}\{|\beta-\mu_{2i-1}|^2,|\beta-\mu_{2i}|^2\}\cdot \sum_{||\l_i|-2|>\varepsilon} \|P_i u\|^2 \\
&\ge& \varepsilon^3/9\cdot\min_{||\l_i|-2|>\varepsilon}\{|\beta-\mu_{2i-1}|^2,|\beta-\mu_{2i}|^2\}.
\end{eqnarray*}
It follows that there exists $i$ with $||\l_i|-2|>\varepsilon$ such that
\begin{equation}\label{eq: first case conclusion}
\min\{|\mu_{2i-1}-\beta|^2,|\mu_{2i}-\beta|^2\} \ \le \ \frac{9\|E\|^2}{\varepsilon^3}.
\end{equation} 

We now consider the second case of $u$ when \eqref{eq: two cases of u} does not hold, or equivalently
\begin{equation}\label{eq: 2nd case u}
\sum_{||\l_i|-2|\le\varepsilon} \|P_i u\|^2 > 1-\varepsilon.
\end{equation}
We partition the set of indices $i$ satisfying $||\l_i|-2|\le\varepsilon$ as a union of $J$ and $I$, where
$J$ is the set of indices $i$ such that $||\l_i|-2|\le\varepsilon$ and $\max\{|y_{2i-1}^*u_i|,|y_{2i}^*u_i|\} > \varepsilon$, and $I$ is the set of indices $i$ such that $||\l_i|-2|\le\varepsilon$ and $\max\{|y_{2i-1}^*u_i|,|y_{2i}^*u_i|\} \le \varepsilon$. It follows from \eqref{eq: 2nd case u} that at least one of the following inequalities hold:
$$
\sum_{i\in J} \|P_i u\|^2 > \varepsilon, \quad \sum_{i\in I} \|P_i u\|^2 > 1-2\varepsilon.
$$
If the first inequality holds then by \eqref{eq: master bound} and Lemma~\ref{lem: span x} we have
\begin{eqnarray*}
\|E\|^2 &\ge& \sum_{i\in J} T_i\cdot\|P_i u\|^2 \\
&\ge& \sum_{i\in J} \max\left\{|(\beta-\mu_{2i-1})y_{2i-1}^*u_i|^2,|(\beta-\mu_{2i})y_{2i}^*u_i|^2\right\}\cdot \|P_iu\|^2\\
&\ge& \min_{i\in J}\max\left\{|(\beta-\mu_{2i-1})y_{2i-1}^*u_i|^2,|(\beta-\mu_{2i})y_{2i}^*u_i|^2\right\}\cdot \sum_{i\in J}  \|P_iu\|^2\\
&\ge& \varepsilon \cdot \min_{i\in J}\max\left\{|(\beta-\mu_{2i-1})y_{2i-1}^*u_i|^2,|(\beta-\mu_{2i})y_{2i}^*u_i|^2\right\}.
\end{eqnarray*}
Since $\max\{|y_{2i-1}^*u_i|,|y_{2i}^*u_i|\} > \varepsilon$ for $i\in J$, it follows that there exists $i\in J$ such that
\begin{equation}\label{eq: 2nd case conclusion 1}
\min\{|\beta-\mu_{2i-1}|^2,|\beta-\mu_{2i}|^2\} \le \frac{\|E\|^2}{\varepsilon^3}.
\end{equation}

We now assume that the following inequality holds:
\begin{equation}\label{eq: u concentrates on I}
\sum_{i\in I} \|P_i u\|^2 > 1-2\varepsilon.
\end{equation}
This inequality implies that $|\beta|$ is bounded. Indeed, from identities $(H+E)u = \beta u$ and $u=\sum_i \|P_i u\| u_i$ we get
\begin{equation}\label{eq: H+E eigenvector}
\sum_{i} \|P_i u\| H u_i + Eu = \beta\sum_i \|P_iu\|u_i.
\end{equation}
Note that $Hu_i\in \text{span}\{x_{2i-1},x_{2i}\}$ because $u_i\in \text{span}\{x_{2i-1},x_{2i}\}$ and $\{x_{2i-1},x_{2i}\}$ are eigenvectors of $H$. Denote $P_I = \sum_{i\in I}P_i$ and apply $P_I$ to both sides of \eqref{eq: H+E eigenvector}, we have
$$
\sum_{i\in I} \|P_i u\| H u_i + P_IEu = \beta\sum_{i\in I} \|P_iu\|u_i.
$$
If $i\in I$ then $H$ is bounded on $\text{span}\{x_{2i-1},x_{2i}\}$ by Lemma~\ref{lem: H is bounded}. Therefore from \eqref{eq: u concentrates on I} we obtain
\begin{eqnarray*}
(1-2\varepsilon)^{1/2}|\beta| \ \le \ \Big\|\beta\sum_{i\in I}\|P_i u\|u_i\Big\| 
\ \le \ \Big\|\sum_{i\in I} \|P_i u\| H u_i\Big\| + \|P_IEu\| 
\ \le \ C + \|E\|.
\end{eqnarray*}
Since $\varepsilon\le1/4$ and $\|E\|\le 1$, this implies $|\beta|\le 2C$. Applying $P_{I^c} = \sum_{i\not\in I}P_i$ to both sides of \eqref{eq: H+E eigenvector}, using \eqref{eq: u concentrates on I} and the boundedness of $\beta$, we have
\begin{eqnarray}\label{eq: H bounded on I^c}
\Big\|\sum_{i\in I^c} \|P_i u\| H u_i \Big\| \le \|P_{I^c} Eu \| + |\beta|\cdot \Big\|P_{I^c}\sum_{i\in I^c} \|P_iu\|u_i\Big\| \le \|E\| + C\sqrt{2\varepsilon}.
\end{eqnarray}
Therefore using $(H+E)u=\beta u$ and inequalities \eqref{eq: u concentrates on I}, \eqref{eq: H bounded on I^c} we have 
\begin{eqnarray}
\nonumber\Big\|\beta u - (H+E)\sum_{i\in I} \|P_iu\| u_i \Big\|
\nonumber&=& \Big\| \sum_{i\in I^c}\|P_iu\| H u_i + E\sum_{i\in I^c} \|P_iu\| u_i \Big\| \\
\nonumber&\le& (\|E\|+C\sqrt{2\varepsilon}) + \|E\| \\
\label{eq: truncate  u}&\le& 2C(\sqrt{\varepsilon}+\|E\|). 
\end{eqnarray}
Denote $\bar{x}_{2i-1} = \|x_{2i-1}\|^{-1} x_{2i-1}$ and $\bar{x}_{2i} = \|x_{2i}\|^{-1} x_{2i}$. Since $\bar{x}_{2i-1}\perp y_{2i}$,  $\bar{x}_{2i}\perp y_{2i-1}$ and $\max\{|y_{2i-1}^*u_i|,|y_{2i}^*u_i|\} \le \varepsilon$ for $i\in I$, it follows that 
$|\ip{u_i}{\bar{x}_{2i-1}}| \ge 1-2\varepsilon$ and $|\ip{u_i}{\bar{x}_{2i}}| \ge 1-2\varepsilon$. 
By multiplying $\bar{x}_{2i}$ with a complex number of magnitude one if necessary, we may assume that $\ip{u_i}{\bar{x}_{2i}}\ge 1-2\varepsilon$ for $i\in I$, and consequently
\begin{equation}\label{eq: ui xbar close}
\|u_i-\bar{x}_{2i}\|^2\le 4\varepsilon.
\end{equation} 
We are now ready to show that $\beta$ is close to an eigenvalue of $H$. 
By \eqref{eq: ui xbar close}, \eqref{eq: u concentrates on I}, \eqref{eq: truncate  u}, the fact that $\beta$ and $\mu_{2i}$ are bounded for $i\in I$, and triangle inequality we have
\begin{eqnarray*}
\Big\|\sum_{i\in I} \|P_i u\|(\mu_{2i}-\beta) u_i\Big\| &=& \Big\| \sum_{i\in I} \|P_i u\|\mu_{2i} u_i -  \sum_{i\in I} \|P_i u\|\beta u_i \Big\|\\
&\le& \Big\| \sum_{i\in I} \|P_i u\|\mu_{2i} \bar{x}_{2i} -  \sum_{i\in I} \|P_i u\|\beta u_i \Big\| + C\sqrt{4\varepsilon} \\
&\le& \Big\| \sum_{i\in I} \|P_i u\|\mu_{2i} \bar{x}_{2i} -  \sum_{i=1}^n \|P_i u\|\beta u_i \Big\| + C(\sqrt{4\varepsilon}+\sqrt{2\varepsilon}) \\
&=& \Big\| H \sum_{i\in I} \|P_i u\|\ \bar{x}_{2i} -  \beta u \Big\| +C(\sqrt{4\varepsilon}+\sqrt{2\varepsilon}) \\
&\le& \Big\| H \sum_{i\in I} \|P_i u\|\ u_{i} -  \beta u \Big\| +C(2\sqrt{4\varepsilon}+\sqrt{2\varepsilon}) \\
&\le& \Big\| (H+E) \sum_{i\in I} \|P_i u\|\ u_{i} -  \beta u \Big\| +C(2\sqrt{4\varepsilon}+\sqrt{2\varepsilon}) +\|E\| \\
&\le&2C(\sqrt{\varepsilon}+\|E\|) + C(2\sqrt{4\varepsilon}+\sqrt{2\varepsilon}) +\|E\| \\
&\le& 8C(\sqrt{\varepsilon} +\|E\|).
\end{eqnarray*}
Together with \eqref{eq: u concentrates on I} this implies  
\begin{equation}\label{eq: 2nd case conclusion 2}
\min_{i\in I} |\beta-\mu_{2i}|^2 \le \frac{1}{1-2\varepsilon}\cdot\sum_{i\in I} \|P_i u\|^2|\beta-\mu_{2i}|^2 \le C(\varepsilon+\|E\|^2).
\end{equation}
Finally, it follows from \eqref{eq: first case conclusion}, \eqref{eq: 2nd case conclusion 1} and \eqref{eq: 2nd case conclusion 2} that if $\beta$ is an eigenvalue of $H+E$ then there exists an eigenvalue $\mu$ of $H$ such that 
$$
|\beta - \mu| \ \le \ \frac{C(\|E\|+\varepsilon^2)}{\varepsilon^{3/2}}  \ = \ 2C\|E\|^{1/4}
$$ 
for $\varepsilon = \|E\|^{1/2}$. It follows from Assumption~\ref{as: concentration of adjacency matrix} that $\|E\| = O(1/\sqrt{d})$ and therefore the proof is complete.
\end{proof}

\begin{proof}[Proof of Lemma~\ref{lem: nonsingularity bound}]
We prove the first inequality:
\begin{eqnarray*}
\|ax+by\|^2 &=& |a|^2 + |b|^2 +2\cdot\text{Re}\{\bar{a}b\ip{x}{y}\} \\
&\ge& |a|^2 + |b|^2 - 2|ab|(1-\varepsilon) \\
&=& (1-\varepsilon)(|a|-|b|)^2 + \varepsilon(|a|^2+|b|^2)\\
&\ge& \varepsilon(|a|^2+|b|^2).
\end{eqnarray*}
To prove the second inequality, denote $z = x-y$ and $w=x+y$. Then $z,w$ are perpendicular and $x = (z+w)/2$, $y = (w-z)/2$. Therefore
\begin{eqnarray*}
|\ip{v}{x}|^2 + |\ip{v}{y}|^2 = v^*(xx^*+yy^*)v 
=v^*(zz^*+ww^*)v/2.
\end{eqnarray*}
Note that the restriction of $zz^*+ww^*$ on $\text{span}\{x,y\}$ is a positive definite matrix with eigenvalues $\|z\|^2$ and $\|w\|^2$ because $z$ and $w$ are perpendicular. By the first inequality
$$
\min\{\|z\|^2,\|w\|^2\} = \min\{\|x-y\|^2,\|x+y\|^2\}
\ge 2\varepsilon.$$
Since $v\in\text{span}\{x,y\}$, it follows that
$$
v^*(zz^*+ww^*)v/2 \ge 2\varepsilon v^*v/2 = \varepsilon.
$$
The proof is complete.
\end{proof}

\begin{proof}[Proof of Lemma~\ref{lem: span x}]
We decompose $x_{2i-1}$ as $x_{2i-1} = z + w$ where $z\perp x_{2i}$ and $w \in \text{span}\{x_{2i}\}$. Then
$$
\|ax_{2i-1}+bx_{2i}\|^2 = |a|^2 \|z\|^2 + \|aw+bx_{2i}\|^2 \ge |a|^2 \|z\|^2.
$$
To calculate $z$, denote $\bar{x}_{2i-1} = \|x_{2i-1}\|^{-1} x_{2i-1}$, $\bar{x}_{2i} = \|x_{2i}\|^{-1} x_{2i}$ and $\tau = \ip{\bar{x}_{2i-1}}{\bar{x}_{2i}}$. 
From \eqref{eq: x inner product} and \eqref{eq: norm of x} we get
$$
\tau = 
\begin{cases}
 -\frac{\l_i^2+\l_i\sqrt{\l_i^2-4}}{4}, & \text{ if } |\l_i| < 2 \\
-\frac{2}{|\l_i|}, & \text{ if } |\l_i| \ge 2.
\end{cases}
$$
Since $\|x_{2i-1}\| = \|x_{2i}\|$, it follows that
$$
z = x_{2i-1} - \ip{x_{2i-1}}{\bar{x}_{2i}} \bar{x}_{2i} =  x_{2i-1} - \tau x_{2i}.
$$
Therefore by \eqref{eq: norm of x}, we obtain
\begin{eqnarray*}
\|z\|^2 &=& \|x_{2i-1}\|^2 + |\tau|^2 \|x_{2i}\|^2 - 2\text{Re}(\tau\ip{x_{2i-1}}{x_{2i}}) \\
&=& \|x_{2i}\|^2\left( |\tau|^2 + 1 - 2\text{Re}(\tau^2) \right)\\
&=& 
\begin{cases}
\frac{4}{4-\l_i^2}\left(\frac{\l_i^2}{4}+1-\frac{\l_i^4-2\l_i^2}{4}\right),&\text{if }
|\l_i|<2 \\
\frac{\l_i^2}{\l_i^2 - 4}\left(1-\frac{4}{\l_i^2}\right), & \text{if }
|\l_i|\ge 2
\end{cases}\\
&=& 
\begin{cases}
\l_i^2+1,&\text{if }
|\l_i|<2 \\
1, & \text{if }
|\l_i|\ge 2
\end{cases}\\
&\ge& 1.
\end{eqnarray*}
This implies $\|ax_{2i-1}+bx_{2i}\| \ge |a|\cdot\|z\|\ge |a|$. By decomposing $x_{2i}$ instead of $x_{2i-1}$ and repeating the same argument, we obtain $\|ax_{2i-1}+bx_{2i}\| \ge |b|$. The proof is complete.
\end{proof}

\begin{proof}[Proof of Lemma~\ref{lem: H is bounded}]
Since $\bar{x}_{2i-1}=\|x_{2i-1}\|^{-1} x_{2i-1}$ and $y_{2i}$ form an orthonormal basis of $W_i$, it is enough to bound $\|H\bar{x}_{2i-1}\|$ and $\|Hy_{2i}\|$. Note that the restriction $H_i$ of $H$ on $W_i$ has the formula
$$
H_i = \mu_{2i-1}x_{2i-1}y^*_{2i-1}+\mu_{2i}x_{2i}y^*_{2i}.
$$
Therefore $\|H_i\bar{x}_{2i-1}\| = \|\mu_{2i-1}\bar{x}_{2i-1}\|\le |\l_i|$. For the more involved calculation of $Hy_{2i}$ we will repeatedly use identities 
\begin{equation}\label{eq: mu identities}
\mu_{2i-1}\mu_{2i}=1, \quad \mu_{2i-1}+\mu_{2i}=\l_i 
\end{equation}
which follow directly from the formulas of $\mu_{2i-1}$ and $\mu_{2i}$ in \eqref{eq: mu formula}. 

\medskip
\noindent{\bf The case $|\l_i|<2$.}
From \eqref{eq: y inner product}, \eqref{eq: x formula} and identities $|\mu_{2i-1}|=|\mu_{2i}|=1$, $\mu_{2i-1}\mu_{2i}=1$ we have
\begin{eqnarray*}
H_i y_{2i} &=& \frac{\l_i\mu_{2i-1}^2}{\sqrt{2}(\mu_{2i}-\mu_{2i-1})}
\begin{pmatrix}
v_i\\
\mu_{2i}v_i
\end{pmatrix}
+ \frac{\sqrt{2}\mu_{2i}}{\mu_{2i-1}-\mu_{2i}} 
\begin{pmatrix}
v_i\\
\mu_{2i-1}v_i
\end{pmatrix} \\
&=& \frac{1}{\sqrt{2}(\mu_{2i}-\mu_{2i-1})}
\begin{pmatrix}
(\l_i\mu_{2i-1}^2-2\mu_{2i})v_i\\
(\l_i\mu_{2i-1} - 2)v_i
\end{pmatrix}.
\end{eqnarray*}
Using \eqref{eq: mu identities} we get
\begin{eqnarray*}
\l_i\mu_{2i-1}^2 - 2\mu_{2i} &=& (\mu_{2i-1}+\mu_{2i})\mu_{2i-1}^2 - 2\mu_{2i} \\
&=& \mu_{2i-1}^3+\mu_{2i-1} -2\mu_{2i}\\
&=& (\mu_{2i-1}-\mu_{2i})(\mu_{2i-1}^2+1).
\end{eqnarray*}
Similarly, 
$$\l_i\mu_{2i-1} - 2 = (\mu_{2i-1}+\mu_{2i})\mu_{2i-1} - 2 = \mu_{2i-1}^2 - 1 = \mu_{2i-1}(\mu_{2i-1}-\mu_{2i}).$$
Therefore 
$$
\|Hy_{2i}\|^2 = (|\mu_{2i-1}^2+1|^2+|\mu_{2i-1}|^2)/2 \le 5/2.
$$

\medskip
\noindent{\bf The case $\l_i\ge 2$.}
In this case $\mu_{2i-1}$ and $\mu_{2i}$ are real positive numbers. Then from \eqref{eq: y inner product}, \eqref{eq: x formula} and \eqref{eq: mu identities}  we have
\begin{eqnarray*}
H_i y_{2i} &=& 
\frac{2\mu_{2i-1}\sqrt{1+\mu_{2i-1}^2}}{\l_i(\mu_{2i}-\mu_{2i-1})} 
\begin{pmatrix}
v_i\\
\mu_{2i}v_i
\end{pmatrix}
+ \frac{\mu_{2i}\sqrt{1+\mu_{2i}^2}}{\mu_{2i-1}-\mu_{2i}} 
\begin{pmatrix}
v_i\\
\mu_{2i-1}v_i
\end{pmatrix}.
\end{eqnarray*}
It follows from \eqref{eq: mu identities} that 
$$\sqrt{1+\mu_{2i-1}^2} = \mu_{2i-1}\sqrt{1+\mu_{2i}^2}.$$   
Therefore
\begin{eqnarray*}
H_i y_{2i} = 
\frac{\sqrt{1+\mu_{2i}^2}}{\l_i(\mu_{2i}-\mu_{2i-1})} 
\begin{pmatrix}
(2\mu_{2i-1}^2-\l_i\mu_{2i})v_i\\
(2\mu_{2i-1}-\l_i)v_i
\end{pmatrix}
= 
-\frac{\sqrt{1+\mu_{2i}^2}}{\l_i} 
\begin{pmatrix}
(\mu_{2i-1}+\l_i)v_i\\
v_i
\end{pmatrix}.
\end{eqnarray*}
Note that $\mu_{2i}\le 1$ and $\mu_{2i-1}\le \l_i$ by \eqref{eq: mu formula}. Hence
$$
\|H_iy_{2i}\|^2 = \frac{(1+\mu_{2i}^2)(1+(\mu_{2i-1}+\l_i)^2)}{\l_i^2} \le 10. 
$$

\medskip
\noindent{\bf The case $\l_i\le -2$.}
In this case $\mu_{2i-1}$ and $\mu_{2i}$ are real negative numbers. Then from \eqref{eq: y inner product}, \eqref{eq: x formula} and \eqref{eq: mu identities}  we have
\begin{eqnarray*}
H_i y_{2i} &=& 
\frac{2\mu_{2i-1}\sqrt{1+\mu_{2i-1}^2}}{\l_i(\mu_{2i-1}-\mu_{2i})} 
\begin{pmatrix}
v_i\\
\mu_{2i}v_i
\end{pmatrix}
+ \frac{\mu_{2i}\sqrt{1+\mu_{2i}^2}}{\mu_{2i-1}-\mu_{2i}} 
\begin{pmatrix}
v_i\\
\mu_{2i-1}v_i
\end{pmatrix}.
\end{eqnarray*}
It follows from \eqref{eq: mu identities} that 
$$\sqrt{1+\mu_{2i-1}^2} = -\mu_{2i-1}\sqrt{1+\mu_{2i}^2}.$$   
Therefore
\begin{eqnarray*}
H_i y_{2i} = 
\frac{\sqrt{1+\mu_{2i}^2}}{\l_i(\mu_{2i}-\mu_{2i-1})} 
\begin{pmatrix}
(2\mu_{2i-1}^2-\l_i\mu_{2i})v_i\\
(2\mu_{2i-1}-\l_i)v_i
\end{pmatrix}
= 
-\frac{\sqrt{1+\mu_{2i}^2}}{\l_i} 
\begin{pmatrix}
(\mu_{2i-1}+\l_i)v_i\\
v_i
\end{pmatrix}.
\end{eqnarray*}
Note that $\mu_{2i}^2\le \l_i^2$ and $|\mu_{2i-1}|\le 1$ by \eqref{eq: mu formula}. Hence
$$
\|H_iy_{2i}\|^2 = \frac{(1+\mu_{2i}^2)(1+(\mu_{2i-1}+\l_i)^2)}{\l_i^2} \le 10\l_i^2. 
$$
The proof is complete.
\end{proof}

\section{Proof of Theorem~\ref{thm: consistency BH method}}

\begin{proof}[Proof of Theorem~\ref{thm: consistency BH method}]
We first rewrite the Bethe Hessian as follows:
\begin{equation*}
  H(r) = (r^2-1)I - r(A-\E A) + D - r\bar{A} =: \hat{H}(r) - r\E A.
\end{equation*}
We show that eigenvalues of $\hat{H}(r)$ are non-negative and are of smaller order than non-zero eigenvalues of $r\E A$.
This in turn implies that $K$ eigenvalues of $H(r)$ are negative while the rest are positive.

By Theorem~\ref{as: concentration of adjacency matrix}, with probability at least $1-1/n$ we have
\begin{equation}\label{eq: concentration bound}
  \|A-\E A\| \le 2\sqrt{d} + C\sqrt{\log n}.
\end{equation}
To bound the node degrees, we use the standard Bernstein's inequality: with probability at least $1-1/n$,
\begin{equation}\label{eq: degree bound}
  \|D - \E D\| \le C\sqrt{d\log n}, \quad |r^2 - (1+\varepsilon)^2 d| \le C\sqrt{d\log n}.
\end{equation}
For square matrices $X,Y$ we use $X\succeq Y$ to signify that $X-Y$ is positive semidefinite.
Then by \eqref{eq: concentration bound}, \eqref{eq: degree bound}
and Assumption~\ref{as: lowrank EA}, we have
\begin{eqnarray}\label{eq: positive definite}
\nonumber  \hat{H}(r) &\succeq& \left[(r^2 -1) - r\left(2\sqrt{d} + C\sqrt{\log n}\right) + (1+\varepsilon)^2d - C\sqrt{d\log n}\right] I \\
\nonumber  &\succeq& \left[ \left(r-\sqrt{d}\right)^2 + (2\varepsilon+\varepsilon^2) d - C\sqrt{d\log n} \right] I \\
  &\succeq& 0
\end{eqnarray}
because  $\varepsilon = C\sqrt{\log n/d}$.

For a subspace $U \subseteq \R^n$, we denote by $\mathrm{dim}(U)$ the dimension of $U$, and by $U^\perp$ the orthogonal complement of $U$.
Also, let $\mathrm{col}(\E A)$ be the column space of $\E A$.
Using the Courant min-max principle (see e.g. \cite[Corollary~III.1.2]{Bhatia1996}) and \eqref{eq: positive definite}, we have
$$
\rho_{n-K}(H(r)) = \max_{\mathrm{dim}(U)=n-K} \ \min_{x \in U,\|x\|=1} \langle H(r)x,x \rangle
  \ge  \min_{x \in \mathrm{col}(\E A)^\perp, \|x\|=1} \langle H(r)x,x \rangle \ge 0.
$$
Therefore the $n-K$ largest eigenvalues of $H(r)$ are non-negative.

It remains to show that the $K$ smallest eigenvalues of $H(r)$ are negative.
From \eqref{eq: concentration bound}, \eqref{eq: degree bound}, and a triangle inequality,
we have
\begin{equation}\label{eq: upper bound of Hhat}
  \|\hat{H}(r)\| \le 4 d + C\sqrt{d\log n}.
\end{equation}
On the other hand, from \eqref{eq: degree bound} and Assumption~\ref{as: lowrank EA} we get
\begin{equation}\label{eq: upper bound eigenvalue Abar}
   \l_{K}(r\E A) \ge (1+\varepsilon)\sqrt{d}\left(4\sqrt{d}+C\sqrt{\log n}\right) \ge 4d + C\sqrt{d\log n}.
\end{equation}
Combining \eqref{eq: upper bound of Hhat}, \eqref{eq: upper bound eigenvalue Abar},
and using the Courant min-max principle again, we conclude that the $K$ smallest eigenvalues of $H(r)$ are negative, which completes the proof.
\end{proof}

\bibliography{allref}
\bibliographystyle{abbrv}
\end{document}